\documentclass[conference]{IEEEtran}
\usepackage{times}

\usepackage{multicol}
\usepackage{multirow}
\usepackage[bookmarks=true]{hyperref} 
\usepackage{amsmath}
\usepackage{amssymb}
\usepackage{amsthm}
\usepackage[pdftex]{graphicx}   
\usepackage{url}
\usepackage{algorithm2e}
\usepackage{color, colortbl, xcolor}
\usepackage{amsfonts}
\usepackage{dsfont}
\usepackage{amsmath}
\usepackage{multicol}
\usepackage{graphicx}
\usepackage{caption}

\usepackage{subcaption}
\usepackage{balance} 
\usepackage{threeparttable} 

\usepackage[%
    sorting=none,
    natbib=true,  
    backend=biber,
    sortcites=true,
    doi=false,
    url=false,
    giveninits=true,
    maxcitenames=4, 
    minbibnames=3, 
    maxbibnames=4, 
    hyperref
]{biblatex}

\newbibmacro{string+doiurl}[1]{
    \iffieldundef{doi}{
        \iffieldundef{url}{#1}{\href{\thefield{url}}{#1} }
    }{\href{https://dx.doi.org/\thefield{doi}}{#1}}
}
\DeclareFieldFormat{title}{\usebibmacro{string+doiurl}{\mkbibemph{#1}}}
\DeclareFieldFormat[article,incollection,techreport,inproceedings,book,misc]%
    {title}%
    {\usebibmacro{string+doiurl}{\mkbibquote{#1}}}

\renewbibmacro{in:}{}
\AtEveryBibitem{%
    \clearlist{language}%
    \clearfield{isbn}%
    \clearfield{issn}
}

\newtheorem{theorem}{Theorem}
\newtheorem{proposition}{Proposition}

\newtheorem*{example}{Running Example}

\usepackage{ifthen}
\newboolean{include-notes}
\setboolean{include-notes}{true}

\usepackage{xcolor}
\newcommand{\jfnote}[1]%
    {\textcolor{orange}{\textbf{Jaime: #1}}}
\newcommand{\vrnote}[1]%
    {\textcolor{green}{\textbf{Vicen\c{c}: #1}}}
\newcommand{\kcnote}[1]%
    {\textcolor{brown}{\textbf{KC: #1}}}
\newcommand{\ctnote}[1]%
    {\textcolor{purple}{\textbf{Claire: #1}}}
\newcommand{\remove}[1]%
    {\textcolor{red}{#1}}

\usepackage{lipsum}

\newcommand{\state}{{s}}
\newcommand{\ctrl}{{u}}

\newcommand{\traj}{{\xi}}
\newcommand{\csig}{{\mathbf{u}}}

\newcommand{\xset}{{\mathcal{S}}}  
\newcommand{\cset}{{\mathcal{U}}}

\newcommand{\csigset}{{\mathbb{U}}}

\newcommand{\dyn}{{f}}

\newcommand{\fdisc}{{\mathbf{f}}}

\newcommand{\outcome}{{\mathcal{V}}}
\newcommand{\valfunc}{{V}}
\newcommand{\qfunc}{{Q}}

\newcommand{\consfunc}{{g}}
\newcommand{\targfunc}{{l}}

\newcommand{\safeset}{{\Omega}}

\newcommand{\target}{{\mathcal{T}}}
\newcommand{\constraint}{{\mathcal{K}}}
\newcommand{\failure}{{\mathcal{F}}}
\newcommand{\reach}{{\mathcal{R}}}

\newcommand{\reachavoid}{{\mathcal{RA}}}

\newcommand{\reachavoidDiscount}{{\reachavoid_\gamma}}

\begin{document}

\title{
Safety and Liveness Guarantees through Reach-Avoid Reinforcement Learning
}


\author{
    \IEEEauthorblockN{Kai-Chieh Hsu\IEEEauthorrefmark{1}\textsuperscript{\textsection},
    Vicen\c{c} Rubies-Royo\IEEEauthorrefmark{2}\textsuperscript{\textsection},
    Claire J. Tomlin\IEEEauthorrefmark{2},
    Jaime F. Fisac\IEEEauthorrefmark{1}}
    \IEEEauthorblockA{\IEEEauthorrefmark{1}Department of Electrical and Computer Engineering, Princeton University, United States
    \\ \texttt{\{\href{mailto:kaichieh@princeton.edu}{kaichieh},
                 \href{mailto:jfisac@princeton.edu}{jfisac}\}@princeton.edu}}
    \IEEEauthorblockA{\IEEEauthorrefmark{2}Department of Electrical Engineering and Computer Sciences, University of California, Berkeley, United States
    \\ \texttt{\{\href{mailto:vrubies@berkeley.edu}{vrubies},
                 \href{mailto:tomlin@berkeley.edu}{tomlin}\}@berkeley.edu}}
}

%

\maketitle
\begingroup\renewcommand\thefootnote{\textsection}
\footnotetext{Denotes equal contribution in alphabetical order.}
\endgroup

\begin{abstract}

Reach-avoid optimal control problems, in which the system must reach certain goal conditions while staying clear of unacceptable failure modes, are central to safety and liveness assurance for autonomous robotic systems, but their exact solutions are intractable for complex dynamics and environments. Recent successes in the use of reinforcement learning methods to approximately solve optimal control problems with performance objectives make their application to certification problems attractive; however, the Lagrange-type objective (cumulative costs or rewards over time) used in reinforcement learning is not suitable to encode temporal logic requirements. Recent work has shown promise in extending the reinforcement learning machinery to safety-type problems, whose objective is not a sum, but a minimum (or maximum) over time. In this work, we generalize the reinforcement learning formulation to handle all optimal control problems in the reach-avoid category. We derive a time-discounted reach-avoid Bellman backup with contraction mapping properties and prove that the resulting \emph{reach-avoid Q-learning} algorithm converges under analogous conditions to the traditional Lagrange-type problem, yielding an arbitrarily tight conservative approximation to the reach-avoid set. We further demonstrate the use of this formulation with deep reinforcement learning methods, retaining zero-violation guarantees by treating the approximate solutions as untrusted oracles in a model-predictive supervisory control framework. We evaluate our proposed framework on a range of nonlinear systems, validating the results against analytic and numerical solutions, and through Monte Carlo simulation in previously intractable problems.
Our results open the door to a range of learning-based methods for safe-and-live autonomous behavior, with applications across robotics and automation.
See \url{https://github.com/SafeRoboticsLab/safety_rl} for code and supplementary material.
\end{abstract}

\IEEEpeerreviewmaketitle


\section{Introduction} \label{sec:intro}
\begin{figure}[!ht]
    \centering
    \includegraphics[width=0.75\columnwidth]{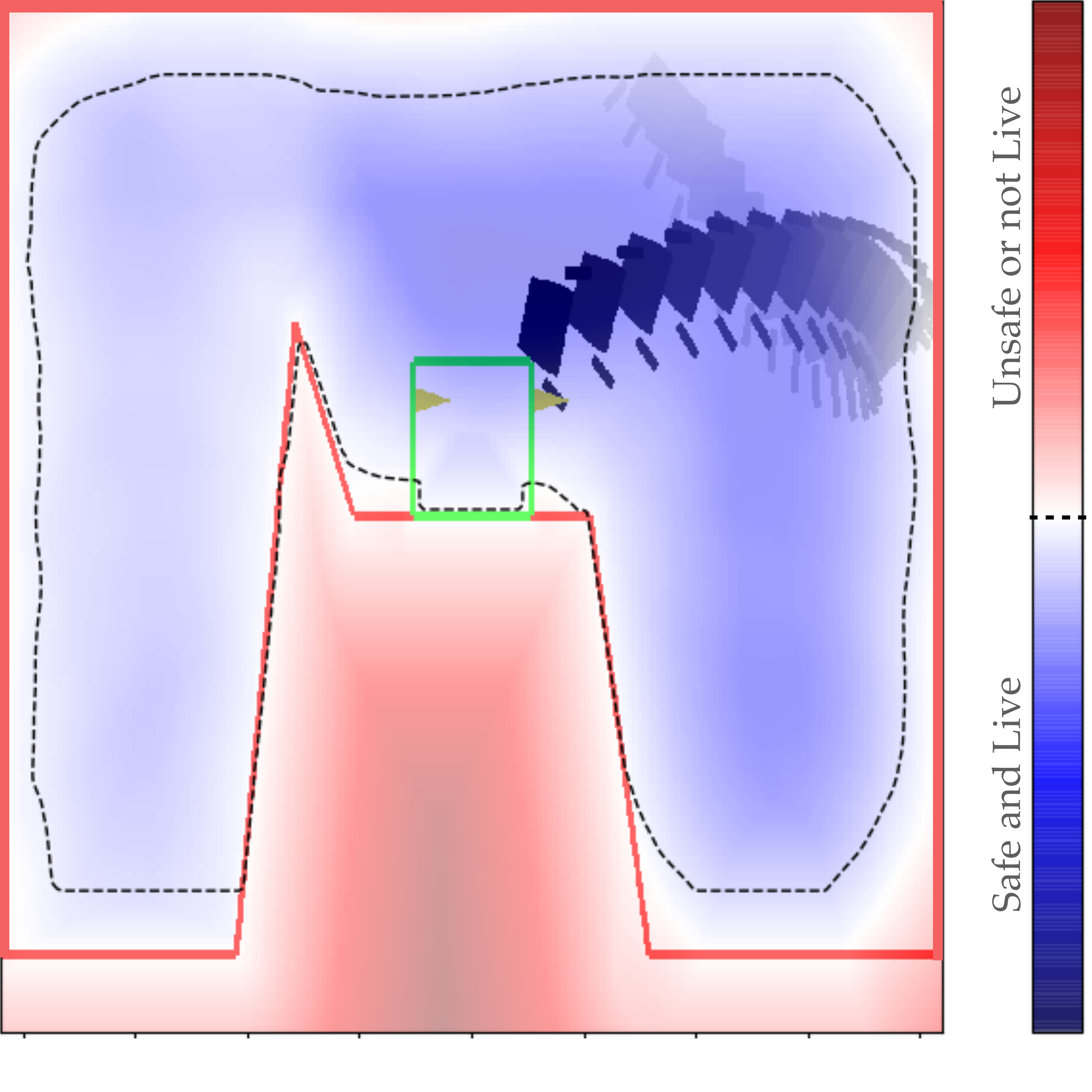}
    \caption{Snapshots of the OpenAI Gym Lunar Lander benchmark system executing the control policy learned through Reach-Avoid Double Deep Q-Network: the vehicle avoids the failure set (terrain and screen edges) and reaches the square (green) target region. The overlaid value function slice and corresponding zero level set (dashed line) indicate the computed safe-and-live positions for the vehicle's initial attitude and velocities.}
    \label{fig:front_fig}
\end{figure}
Recent years have seen a significant increase in the capabilities of robotic systems, driven by rapid advances in sensing and computing hardware paired with novel optimization algorithms and data-driven methodologies.
These improvements have led to an unprecedented growth in the number and variety of robots operating autonomously in the world,
from autonomous airborne delivery of medical supplies~\citep{ackerman2019zipline} to the first driverless vehicle services becoming available to the public in the last year~\citep{siddiquifaiz2020waymo}.

The correct functioning of the decision-making components governing the behavior of these systems is an integral part of their safety and reliability, and indeed their incorrect functioning has led to high-impact system-wide failures,
including
a number of recent deadly accidents involving malfunctioning advanced driver assistance systems~\citep{ntsb2017tesla16,ntsb2020tesla18,ntsb2020tesla19}.
Yet, while ensuring the correct functioning of robotic systems is critical to their viable deployment in high-stakes settings, these are precisely the settings in which it is often most challenging to provide strong operational assurances, because they tend to involve particularly complex, dynamic, and uncertain environments.
This is not made easier by the fact that much of the recent progress on data-driven decision-making yields ``black box'' components whose operation after (or during) training cannot readily be certified, requiring computationally intensive \emph{post hoc} verification~\citep{liu2021algorithms} or supervisory control policies to restrict their outputs at runtime~\citep{fisac2018general, li2020robust}.
Tractably computing decision-making policies that can enforce meaningful properties for autonomous robotic operation is therefore an important open problem in the field.


While the last half-decade has seen increasing interest in guaranteeing safety (preventing undesired conditions, such as collisions or severe rule violations) in robot planning and learning-based control, comparatively little attention has been paid to ensuring liveness,
that is, the eventual completion of specified goals (such as delivering passengers or cargo at the desired location).
Most real-world systems must meet a combination of safety and liveness properties, since enforcing one without the other tends to result in undesired behavior (a robot can typically preserve safety by indefinitely remaining at rest, or more easily achieve liveness---at least in theory---by violating safety constraints).
These two properties are highly complementary, and in fact they are known to jointly encode arbitrarily sophisticated temporal logic specifications~\citep{alpern1985defining}.

Reachability analysis is a central mathematical tool enabling the synthesis and certification of controllers with safety and liveness properties for continuous-state dynamical systems.
Given a dynamical system whose evolution can be influenced through a control input, it can determine the set of initial conditions \emph{and} the appropriate control policy for which the state of the system will be driven to some desired \emph{target} configuration while avoiding undesired \emph{failure} conditions.
Specifically, a reach-avoid problem can be seen as the conjunction of two reachability problems: one in the positive, where the controller's goal is to reach the target set, and one in the negative, where the controller seeks to \emph{not} reach the failure set.
However, because the optimal choice of control input is coupled between the two problems, it is \emph{not} enough to solve each of them separately and then combine the computed solutions.

Hamilton-Jacobi analysis~\citep{bansal2017hamiltonjacobi} provides a general methodology to compute the optimal solution to reach-avoid problems and, thereby, robust certificates on the controlled behavior of dynamical systems.
While the theory applies to general nonlinear dynamics, the central \emph{value function} that encodes the optimal controller is the solution of a Hamilton-Jacobi partial differential equation (or in some cases a variational inequality) that, with few exceptions, cannot be obtained analytically in closed form. As a result, general Hamilton-Jacobi methods ultimately rely on state space discretization and are computationally intense, requiring computation and memory exponential in the state dimension, which limits their practical use to systems with up to about 5 continuous state variables~\citep{mitchell2008flexible}.

Given the computation requirements, these synthesis methods are not commonly run online; instead, the most frequent approach is to precompute a controller (and its associated certificate) offline and store it in memory for online lookup. This approach has been successfully applied in a number of different settings, from robust trajectory tracking~\citep{Herbert2017} to safe learning-based control~\cite{fisac2018general}.

Recent efforts have sought to use deep learning to find approximate solutions to reachability problems by directly trying to enforce satisfaction of the associated Hamilton-Jacobi partial differential equation or variational inequality ~\cite{rubiesroyo2017recursive, bansal2020deepreach}. These methods work by repeatedly sampling random points over the state space and time horizon and computing the associated local Hamilton-Jacobi equation error. By enforcing this error to be zero, these methods effectively propagate information in backward time from the boundary condition. In this setting, a dynamical model of the system, usually control-affine, is required in order to compute the Hamiltonian and compute the error. In a similar spirit, a different line of work represents reachability controllers as neural network classifiers by directly exploiting the ``bang-bang" nature of optimal reachability controllers for control-affine systems~\cite{royo2019classification}.

A successful recent approach for Hamilton-Jacobi \emph{safety} analysis relies on the contraction mapping property commonly exploited in the reinforcement learning literature to deliver a converging temporal learning scheme for safety problems~\cite{fisac2019bridging}. In this work the authors exploit said contraction mapping using several reinforcement learning frameworks, and most notably, deep Q-learning.

\emph{Contribution.}
Our central theoretical contribution is the derivation of a novel, time-discounted formulation of the reach-avoid optimal control problem that lends itself to reinforcement learning methods thanks to the contraction mapping induced by the associated dynamic programming equation.
Unlike the previously derived safety-only analog, this discounted reach-avoid solution is guaranteed to produce a conservative under-approximation of the set of conditions from which the system can succeed at its task.
We further show that the approximation becomes arbitrarily tight as the time discounting becomes less pronounced.

Importantly, our guarantees can be extended to the use of deep reinforcement learning methods by treating the approximate optimal policy as an untrusted oracle and applying a supervisory control scheme: this insight is key, because it enables us to construct control policies with guaranteed safety and liveness for high-dimensional nonlinear systems that are intractable with classical approaches.
To determine the reliability of the proposed method as a synthesis tool, we present comprehensive validation results against analytic and numerical solutions when these are obtainable, and by exhaustive Monte Carlo simulation in two high-dimensional nonlinear systems for which reach-avoid solutions have, until now, not been obtainable. Fig.~\ref{fig:front_fig} showcases one of the high-dimensional examples we tested.

\section{Background} 
\label{sec:background}

\subsection{Dynamical System Safety and Liveness Analysis}
\label{subsec:dynamical_systems_reach_avoid_analysis}

Consider a dynamical system with state $\state \in \xset \subset \mathbb{R}^n$ and control input $\ctrl \in \cset \subset \mathbb{R}^m$, which evolves according to the ordinary differential equation
\begin{equation}\label{eq:dynamics_cont}
    \dot{\state} = \dyn(\state, \ctrl)
    \enspace.
\end{equation}
We assume the set $\cset$ to be compact and the dynamics $\dyn$ to be bounded and Lipschitz continuous; under these conditions, system trajectories ${\traj:\mathbb{R}_+\to\xset}$ are well defined, and continuous in time, for all measurable control inputs~\cite[Ch.~2]{Coddington1955}.
We use the notation $\traj_\state^\csig(\cdot)$ for a time trajectory starting at state $\state$ under control signal $\csig$.
For discrete-time approximations, we denote the time step $\Delta t>0$.

Safety and liveness specifications for dynamical systems can be succinctly formulated in terms of a \textit{target set} $\target \subset \mathbb{R}^n$, the collection of states where we want to steer trajectories, and a \textit{constraint set} $\constraint \subset \mathbb{R}^n$, the set of allowable states during the evolution of a trajectory. The complement of the constraint set is the \textit{failure set} which we denote as $\failure = \constraint^c$.
By convention, we let $\target$ and $\constraint$ be closed sets; no other assumptions (convexity, connectedness, etc.) are made.

The \emph{safe set} with respect to the failure set $\failure$ (or equivalently the \emph{viability kernel} of the constraint set $\constraint$), is defined as the collection of initial states from which the controller can indefinitely keep the system away from the failure set:
\begin{equation}
\safeset(\failure) \ := \ \{ \state \in \xset ~|~ \exists \csig \in \csigset, \forall\tau \ge 0, \traj_\state^\csig(\tau) \notin \failure \},
\label{eq:enf-back-safe}
\end{equation}
where $\csigset$ represents the collection of all measurable control signals $\csig: \mathbb{R}_+ \to \cset$. 
Conversely, the \emph{backward-reachable set} (or, more succinctly, the \emph{reach set}) of the target $\target$ is
\begin{equation}
\reach(\target) \ := \ \{ \state \in \xset ~|~ \exists \csig \in \csigset, \exists\tau \ge 0, \traj_\state^\csig(\tau) \in \target \}.
\label{eq:enf-back-reach}
\end{equation}

In many practical problems in robotics our goal is a combination of the above: namely, we want the system to reach a target configuration while maintaining safety. This requirement brings forth the notion of the \emph{reach-avoid set}:
\begin{IEEEeqnarray}{rcl}
\reachavoid(\target; \failure) & \ := \ \{ & \state \in \xset ~|~ \exists \csig \in \csigset, \tau \ge 0,
\nonumber \\
& & \traj_\state^\csig(\tau) \in \target \land \forall \kappa \in [0,\tau] ~ \traj_\state^\csig(\kappa) \notin \failure \},
\label{eq:enf-back-reach-avoid}
\end{IEEEeqnarray}
which denotes the set of states for which some control signal exists that can drive the system to $\target$ while avoiding $\failure$ at all prior times.
Note that this set is in general \emph{not} equal to the intersection of the previous two sets, because (a) there may exist states in $\reach(\target)\cap\safeset(\failure)$ from which the controller can \emph{either} drive the system to $\target$ \emph{or} keep it away from $\failure$, but not both (in other words, the control signals $\csig$ exist in both cases but are mutually incompatible), and (b) there may exist states $\state\not\in\safeset(\failure)$ from which the system cannot be indefinitely kept away from $\failure$, but it can \emph{first} be driven to $\target$.
The latter case is important for encoding complex operational requirements involving preconditions: for example, we may wish to prevent an autonomous vehicle from entering an intersection \emph{before} reaching a state from which it is visible to any potential conflicting traffic.

\subsection{Reach-Avoid Analysis via Dynamic Programming}
\label{subsec:HJ_safety_and_liveness}

In order to compute the reach-avoid set it is helpful to define two \textit{implicit surface functions} which are used to delineate sets of states. For our problem of interest we will define two such functions $\targfunc(\cdot),\consfunc(\cdot) : \xset \to \mathbb{R}$, which are Lipschitz-continuous and satisfy the following properties:
\begin{align}
    \targfunc(\state) \leq 0 &\iff \state \in \target \label{eq:implicit_target} \\
    \consfunc(\state) > 0 &\iff \state \in \failure \enspace. \label{eq:implicit_failure}
\end{align}
Using these implicit surface function definitions, let us define the following payoff functional for trajectories of the system in \textit{ discrete-time}\footnote{We follow the discrete-time formulation consistently with the standard reinforcement learning framework; we direct readers to~\cite{fisac2015reach} for the continuous-time analog.}:
\begin{IEEEeqnarray}{rl}
\label{eq:cost_functional_reach_avoid}
    \hspace{-2mm}
    \outcome^{\csig}(\state) = \min_{\tau \in \{0,1,...\}} \max \Big\{ \targfunc(\traj_{\state}^{\csig}(\tau)),
    \max_{\kappa \in \{0,\hdots,\tau\}} \consfunc(\traj_{\state}^{\csig}(\kappa)) \Big\}.
\end{IEEEeqnarray}
The outer maximum in \eqref{eq:cost_functional_reach_avoid} acts as an ``overwriting'' mechanism. Given definition \eqref{eq:implicit_failure}, if a trajectory remains inside $\constraint$ for all time, the inner maximum will be negative, and $\outcome^{\csig}(\state)$ will be negative if and only if the trajectory reaches $\target$. In contrast, if the constraints are ever violated the inner maximum will be positive, which ensures that the overall payoff must also be positive. In other words, the payoff can only be negative if $\target$ is reached without previously violating the constraints. 

The goal from any initial state will be to minimize the payoff functional \eqref{eq:cost_functional_reach_avoid}, which leads to the definition of the \textit{value function} over the entire state space, 
\begin{equation} \label{eq:value_function_reach_avoid}
    \valfunc(\state) = \inf_{\csig \in \csigset} \outcome^{\csig}(\state) \text{.}
\end{equation}
The \textit{sign} of the value function encodes crucial information on the safety and liveness of our problem. Concretely,
\begin{IEEEeqnarray}{c}\label{eq:zero_level_set_characterization}
    \valfunc(\state) \leq 0 \iff \state \in \reachavoid(\target; \failure).
\end{IEEEeqnarray}

It can be shown that the value function must satisfy the fixed-point \textit{reach-avoid Bellman equation} (RABE) below, where $\state_+^\ctrl = \traj_\state^{u}(\Delta t)$. For details on the derivation we refer the reader to Appendix \ref{appendix_fixed_point}.

\begin{IEEEeqnarray}{rl}
\label{eq:fixed_point_non_contractive}
    \valfunc(\state) = \max\Big\{ \consfunc(\state), \ & \min\big\{\targfunc(\state),
    \inf_{\ctrl \in \cset} \valfunc(\state_+^\ctrl )~\big\}~\Big\}.
\end{IEEEeqnarray}

\begin{example}[2-D point particle]
    Consider a point particle with simple motion \emph{\`a la} Isaacs~\cite{Isaacs1965} that is required to reach a (yellow) target box without crossing any of three (magenta) obstacle boxes or leaving its rectangular domain, as shown in Fig. \ref{fig:exp_convergentFamily}.
    The dynamics of the particle are described by the equation
    \begin{IEEEeqnarray}{c}
    \label{eq:particle_dynamics}
        \dot{x} = \ctrl v_x,\
        \dot{y} = v_y,
    \end{IEEEeqnarray}
    where $v_x, v_y > 0$ are horizontal and vertical speed constants and $\ctrl$ is the control input.
    The state is the particle's position, $\state = [x, y]^T$ and the allowed control is discrete, $\ctrl \in \{ -1, 0, 1\}$.
    The dynamics in \eqref{eq:particle_dynamics} describe a particle that moves continually upward at a rate $v_y$ while only being able to control its horizontal speed through $u$.
    
    Here each box (obstacles, boundary and target) is specified by its center position and dimensions, i.e., $(p, L)$, with $\ p = [p_x, p_y]^T$ and $L = [L_x, L_y]^T$.
    The safety margin function is expressed as the point-wise maximum of signed distance functions.
    Thus, $\consfunc(\state)$ is computed by the margin to the boundary ($\consfunc_b(\state)$) and the obstacles ($\consfunc_o(\state)$) as below
    \begin{IEEEeqnarray}{l}
    \consfunc_b(\state) := \max_i \Delta(\state, p^{c_b})[i] - L^{c_b}[i] / 2,\
    \nonumber \\
    \consfunc_o(\state) := \max_j \max_i L^{c_j}[i] / 2 - \Delta(\state, p^{c_j})[i],
    \nonumber \\
    \consfunc(\state) := \max \big\{ \consfunc_b(\state),\ \consfunc_o(\state) \big\},
    \nonumber
    \end{IEEEeqnarray}
    where $\Delta(\state, p) := [ |x - p_x|,\ |y - p_y| ]^T$.
    The target margin function can be expressed as
    \begin{IEEEeqnarray}{c}
    \targfunc(\state) := \max_i \Delta(\state, p^t)[i] - L^t[i] / 2.
    \nonumber
    \end{IEEEeqnarray}
\end{example}

\subsection{Q-Learning and Approximate Dynamic Programming}

Reinforcement learning, also known as direct adaptive optimal control~\cite{sutton1992rlisdaoc}, provides a framework for solving dynamic decision problems represented as Markov Decision Processes (MDPs) where the underlying objective is to maximize a Lagrange-type functional 
\begin{IEEEeqnarray}{c}
\outcome^{\csig} (\state) = \sum_\tau \gamma^\tau r(\state_\tau, \ctrl_\tau),
\end{IEEEeqnarray}
where $r(\state,\ctrl):\mathbb{R}^n \times \mathbb{R}^m \to \mathbb{R}$ is known as the \emph{reward function} and $\gamma\in[0,1)$ is the time-discount parameter. Tabular Q-learning, or asynchronous dynamic programming \cite{Watkins1992a}, is a type of model-free reinforcement learning scheme which does not require knowledge of the dynamics of the system. It maintains a state-action value matrix, known as the Q-values, and updates each entry in the following manner: ${\qfunc (\state, \ctrl) \xleftarrow{\alpha} r(s,u) + \gamma \max_{\ctrl'} \qfunc(s_+^\ctrl, \ctrl')}$, where the operator $\xleftarrow{\alpha}$ implies ${x \xleftarrow{\alpha} y \equiv x \leftarrow x + \alpha(y-x)}$.
Under some mild conditions  \cite{Tsitsiklis1994}, tabular Q-learning converges to the optimal solution, meaning that the optimal policy and state value function are given by ${\ctrl^* := \arg \max_\ctrl \qfunc(\state, \ctrl)}$ and ${\valfunc(\state) = \qfunc(\state, \ctrl^*)}$.
Unfortunately, the discretization of the state space in tabular Q-learning leads to an exponential growth in storage requirements, an issue which is commonly referred to as the \textit{curse of dimensionality}.


Deep reinforcement learning shows potential to alleviate the curse of dimensionality in optimal control problems by using deep neural networks to find approximately optimal value functions and policies \cite{van2015DDQN, Timothy2016DDPG, Haarnoja2018SAC}.
In particular, double deep Q-Network (DDQN) \cite{van2015DDQN} uses two deep neural networks (NNs) to approximate the Q-values.
One NN is called the online network with parameter $w$ and the other network is called the target network with parameter $\tilde{w}$.
The Q-value is updated by ${\qfunc_w (\state, \ctrl) \xleftarrow{\alpha} r(\state,\ctrl) + \gamma \qfunc_{\tilde{w}} (\state_+^\ctrl, \max_{\ctrl'} \qfunc_w (\state_+^\ctrl, \ctrl'))}$.

In \cite{fisac2019bridging}, the authors show how to use reinforcement learning to approximate the solution to the Hamilton-Jacobi-Bellman (HJB) equation for the special case where the cost functional is given in the form ${\outcome^{\csig} = \min_{\tau \in \{0,1,...\}} \ \consfunc(\traj_\state^{\csig})}$. In particular, they use the following Q-learning update 
\begin{IEEEeqnarray}{ll}
\qfunc_w (\state, \ctrl) \leftarrow \
& \gamma \min \big\{ \consfunc(\state),\ \qfunc_{\tilde{w}} (\state_+^\ctrl,\ \max_{\ctrl'} \qfunc_w (\state_+^\ctrl, \ctrl')) \big\}
\nonumber \\
& + (1-\gamma)\ \consfunc(\state),
\label{eq:contract_HJB_safety}
\end{IEEEeqnarray}
which can be used to approximate solutions for safety or liveness problems, but not both. In the next section we introduce a more general update for the broader class of reach-avoid problems.
\section{The Time-Discounted Reach-Avoid Bellman Equation} 
\label{sec:time_discounted_reach_avoid_BE}

\begin{figure*}[t]
    \centering
    \includegraphics[width=0.9\textwidth]{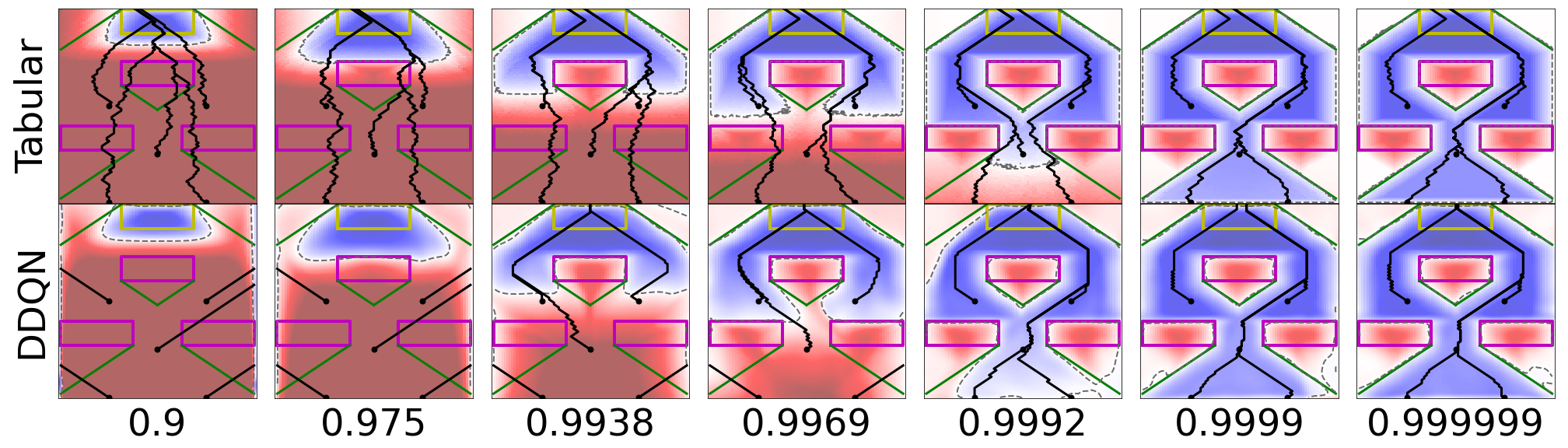}
    \caption{A convergent family of under-approximations that asymptotically approaches the undiscounted reach-avoid set as $\gamma \rightarrow 1$. The red region indicates positive state value, while blue region indicates negative state value. The dashed gray line specifies the zero level set or the discounted reach-avoid set boundary, while Green lines specify analytic reach-avoid set boundary. The solid black lines show trajectory rollouts from five initial states.}
    \label{fig:exp_convergentFamily}
\end{figure*}

Unlike the Lagrange-type dynamic programming equation commonly used in reinforcement learning, the reach-avoid optimality condition \eqref{eq:fixed_point_non_contractive}
does not induce a contraction mapping on $\valfunc$.
This contraction mapping property is a crucial requirement for the convergence of the value iteration scheme and related temporal learning approaches, such as Q-learning~\cite{Watkins1992a,Tsitsiklis1994}.
In what follows we provide a principled modification of \eqref{eq:fixed_point_non_contractive} that induces a contraction mapping in the space of value functions through the introduction of a time discount parameter $\gamma \in [0,1)$. This enables the use of reinforcement learning methods based on temporal learning, in the same spirit as \cite{fisac2019bridging}. In addition, we will see that for any choice of $\gamma$, the fixed-point solution $V_\gamma(x)$ and its corresponding reach-avoid set will be an under-approximation of the true (undiscounted) reach-avoid set $\reachavoid_\gamma \subseteq \reachavoid$, which (under mild technical assumptions) becomes arbitrarily tight as $\gamma\to1$.

\subsection{Probabilistic Interpretation of Time Discounting}
As previously stated, \eqref{eq:fixed_point_non_contractive} is not guaranteed to induce a contraction in the space of value functions. This can be addressed by introducing a discount factor $\gamma \in [0,1)$, which can be interpreted as a probability of episode continuation. In ``traditional'' (Lagrange-type) reinforcement learning, for example, the value of a state can be written using this interpretation as follows
\begin{IEEEeqnarray}{c}
    \valfunc(\state) = \max_{\ctrl \in \cset} ~ 
        (1-\gamma) r(\state, \ctrl) + 
        \gamma \big( r(\state, \ctrl) + \valfunc \big(\state_+^\ctrl\big) \big),
\end{IEEEeqnarray}
with $\state_+^\ctrl = \state + f(\state, \ctrl) \Delta t$.
The term $1-\gamma$ can be seen as representing the probability of the episode terminating in one step, therefore only accounting for the immediate reward, with $\gamma$ conversely representing the probability of the episode continuing, hence accounting for the immediate reward and the expectation of future returns, as encoded in the value function.
Maintaining this interpretation, and noting that truncating \eqref{eq:cost_functional_reach_avoid} to a single-step problem leaves ${\outcome^{\csig}(\state) = \max\{g(\state),l(\state)\}}, \forall\csig$,
we can modify \eqref{eq:fixed_point_non_contractive} with an analogous ``probabilistic time discounting'':
\begin{IEEEeqnarray}{rcl}
    B_\gamma [\valfunc](\state) & \ := \ &
    \gamma \max \Big\{ 
        \min \big\{
            \min_{\ctrl \in \cset} \valfunc( \state_+^\ctrl),\ 
            \targfunc(\state) \big\},\
        \consfunc(\state) \Big\}
    \ +
    \nonumber \\ 
    & &
    (1-\gamma) \max\{ \targfunc(\state), \consfunc(\state) \},
    \nonumber \\
    \valfunc_\gamma(\state) & \ = \ & B_\gamma [\valfunc_\gamma](\state).
    \label{eq:DRABE}
\end{IEEEeqnarray}
We denote \eqref{eq:DRABE} the \textit{discounted reach-avoid Bellman equation} (DRABE).
Unlike its undiscounted counterpart~\eqref{eq:fixed_point_non_contractive}, this fixed-point equation does induce a contraction mapping in the space of value functions. We show this next.

\subsection{Contraction Mapping and Reach-Avoid Q-learning}
\label{sec:contraction_Q_learning}

Based on the time-discounted value function defined in \eqref{eq:DRABE}, we can define the $\gamma$-discounted reach-avoid set, by analogy with~\eqref{eq:zero_level_set_characterization}, as
$\reachavoidDiscount := \{ \state \in \xset | \valfunc_\gamma (\state) \leq 0 \}$.
We present three propositions for the proposed DRABE and leave their proofs in Appendix~\ref{appendix_fixed_point}.
We first show that DRABE induces a contraction mapping.
With the important contraction mapping property at hand, Q-learning can be guaranteed to converge if the learning rate satisfies the assumptions in \cite{Tsitsiklis1994}.
In practice, this contraction mapping property also enables the use of approximate (``deep'') Q-learning methods to be directly applied, as demonstrated in Section \ref{sec:exp}, even though these methods lack the theoretical convergence guarantees of their exact (tabular) counterpart.
Finally, we show that the fixed-point solution of DRABE converges to the fixed-point solution of RABE when $\gamma$ approaches to $1$.

\begin{proposition}[Contraction Mapping]
\label{prop:contraction}
The discounted reach-avoid Bellman equation~\eqref{eq:DRABE} induces a contraction mapping under the supremum norm for any $\gamma \in [0, 1)$.
\end{proposition}
\begin{proposition}[Convergence of Reach-Avoid Q-Learning]\label{prop:ra_q_learning}
    Let ${\mathbf{S}\subseteq\mathcal{S}}$ and $\mathbf{U}\subseteq\mathcal{U}$ be finite discretizations of the state and action spaces, and let $\mathbf{f}:\mathbf{S}\times\mathbf{U}\to\mathbf{S}$ be a discrete transition function approximating the system dynamics.
    The Q-Learning scheme, applied to the Discounted Reach-Avoid Bellman Equation~\eqref{eq:DRABE} and executed on the above discretization converges, with probability $1$, to the optimal state-action safety value function
    \begin{align*}
        Q(\mathbf{s},\mathbf{u}):= & (1-\gamma) \max\{ \targfunc(\mathbf{s}), \consfunc(\mathbf{s}) \} +
        \\
        & \gamma \max \Big\{
        \consfunc(\mathbf{s}),\
        \min\Big\{ \targfunc(\mathbf{s}), \max_{\mathbf{u}'\in\mathbf{U}}Q\big(\fdisc(\mathbf{s}, \mathbf{u}),\mathbf{u}'\big) \Big\} \Big\}
        \enspace,
    \end{align*}
    in the limit of infinite exploration time and given partly-random episode initialization and learning policy with full support over $\mathbf{S}$ and $\mathbf{U}$ respectively.
    Concretely, the temporal difference learning rule is:
    \begin{align*}
        &Q_{k+1}(\mathbf{s},\mathbf{u}) \gets Q_{k}(\mathbf{s},\mathbf{u}) + \alpha_{k} \Big[
        (1-\gamma) \max\{ \targfunc(\state), \consfunc(\state)\}
        \ +
        \nonumber \\ 
        &
        \gamma \max \Big\{ \consfunc(\state) ,
        \min \big\{
            \max_{\mathbf{u}'\in\mathbf{U}}Q\big(\fdisc(\mathbf{s}, \mathbf{u}),\mathbf{u}'\big),\ 
            \targfunc(\state) \big\}\
        \Big\}
        - Q_{k}(\mathbf{s},\mathbf{u}) \Big]
        \,,
    \end{align*}
    for learning rate $\alpha_{k}(\mathbf{s},\mathbf{u})$ satisfying
    \begin{equation*}
        \sum_k \alpha_k(\mathbf{s},\mathbf{u}) = \infty \quad
        \sum_k \alpha_k^2(\mathbf{s},\mathbf{u}) < \infty
        \enspace,
    \end{equation*}
    for all $\mathbf{s}\in\mathbf{S}, \mathbf{u}\in\mathbf{U}$.
    \newline
\end{proposition}
\begin{proposition}[Convergent Value Function]
\label{prop:convergent_value}
In the limit of $\gamma \to 1$, the fixed point of DRABE~\eqref{eq:DRABE} converges to the fixed point of RABE~\eqref{eq:fixed_point_non_contractive} and is the solution of \eqref{eq:value_function_reach_avoid}.
\end{proposition}

\subsection{Discounted Reach-Avoid Sets}

In Section \ref{sec:contraction_Q_learning}, we have shown that Q-learning methods converge to a fixed point.
Another crucial property is the connection between the discounted ($0 \leq \gamma < 1$) and the exact ($\gamma = 1$) reach-avoid set.
The following result shows that the discounted reach-avoid set defined by the discounted reach-avoid value function is always an under-approximation and asymptotically approaches the true reach-avoid set as $\gamma$ approaches $1$.

\begin{theorem}[Conservativeness of Discounted Reach-Avoid Set]
\label{thm:underapproximation_convergent_family}
For any discount factors $0 \leq \gamma_1 \leq \gamma_2 < 1$, the reach-avoid set associated with $\gamma_1$ is a subset of the reach-avoid set associated with $\gamma_2$. 
Moreover, as long as $\reachavoid$ has a nonempty interior and $\valfunc$ is nowhere locally constant on its zero level set, $\reachavoidDiscount$ asymptotically approaches $\reachavoid$, under the Hausdorff set distance, as $\gamma \rightarrow 1$.
\end{theorem}

The proof can be found in Appendix~\ref{appendix_fixed_point}. Theorem \ref{thm:underapproximation_convergent_family} states that any increasing sequence of discount factors approaching $1$ has an associated sequence of reach avoid-sets where each set contains all of its predecessors. In addition, as $\gamma$ tends to $1$, the sequence of discounted reach-avoid sets approaches the true reach-avoid set. These features are particularly appealing for \textit{safe learning} since any discount factor will yield a conservative but safe under-approximation of the reach-avoid set. In other words, the value function should never wrongly predict that it can safely reach the target when in fact it collides. This property makes the value function for any $\gamma$ a good candidate for \textit{shielding} methods~\cite{fisac2018general,bastani2019safe}, where the agent is allowed to explore the environment until safety or liveness are at stake. Thus, any discounted reach-avoid set yields a permissible region of exploration and an associated emergency controller through the value function.  

\begin{example}[2-D point particle]
    Fig. \ref{fig:exp_convergentFamily} demonstrates the convergent family of under-approximate reach-avoid sets obtained by different Q-Learning methods, which are tabular Q-learning (TQ) \cite{Watkins1992a} and double deep Q-network (DDQN) \cite{van2015DDQN}
    \footnote{For TQ, We use a grid of $81 \times 241$ cells, while for DDQN, we describe the architecture and activation in Appendix \ref{appendix_traing}.}.
    First, we look at the discounted reach-avoid sets obtained by TQ as shown in the first row of Fig. \ref{fig:exp_convergentFamily}.
    We observe the boundary of the discounted reach-avoid set is very close to the target set initially.
    However, as $\gamma$ gradually annealing to 1, it approaches the undiscounted reach-avoid set.
    More importantly, the discounted reach-avoid sets are always under-approximations, which means the proposed method guarantees safety and liveness more conservatively.
    Secondly, we compare the reach-avoid sets by DDQN with TQ.
    Similarly, the discounted reach-avoid sets are generally under-approximations with some errors near the boundary.
    This result suggests that deep RL methods manage to approximate reach-avoid games' value function as long as we slowly anneal discount rate to 1.
\end{example}

\section{Results and Validation}
\label{sec:exp}

In this section we present the results of implementing our proposed Discounted Reach-Avoid Bellman Equation on a 2-D point particle problem, 3-D Dubins car problem, a 6-D Lunar Lander problem and a 6-D attack-defense game.
We use DDQN to learn the approximate value function and we describe the architecture, training hyper-parameters and initialization of the NN in Appendix \ref{appendix_traing}.

\subsection{2-D point particle: Sum of Costs versus Reach-Avoid}
This first problem employs the same dynamic system in our running example.
We compare reinforcement learning directly optimizing the discounted reach-avoid cost (RA) against reinforcement learning minimizing the usual sum of discounted costs (Sum).
We consider a sparse cost for the usual reinforcement learning, i.e., $c(\state, \ctrl) = -1$ if $\state_+^\ctrl \in \target$, $c(\state, \ctrl) = \rho$ if $\state_+^\ctrl \in \failure$, and zero elsewhere.
Here, we fix $\gamma$ to $0.95$ in Sum training and $0.9999$ in RA training. Rollouts end when the agent hits the outer square boundary. This termination condition we denote as ``End''.


\begin{figure}[t]
    \centering
    \includegraphics[width=0.9\columnwidth]{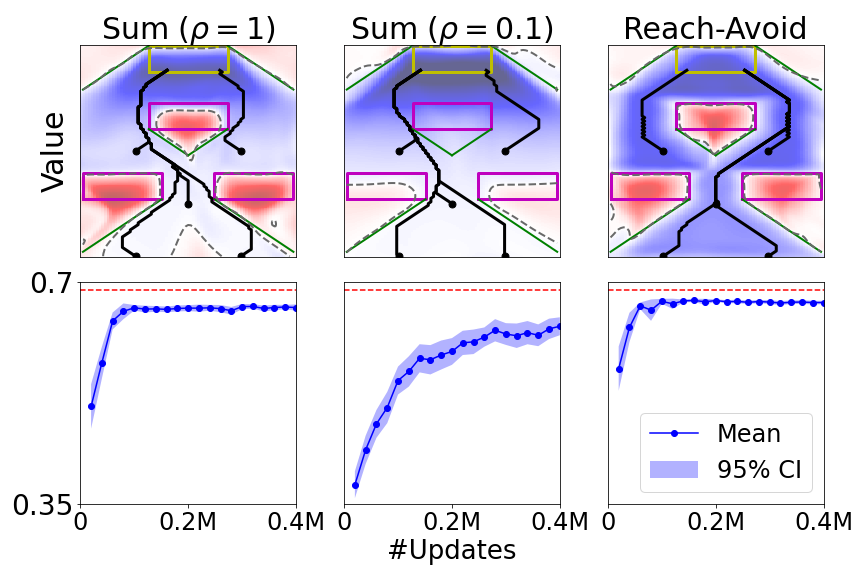}
    \caption{First row: Value function after $400,000$ gradient updates. Second row: Fraction of episodes reaching the target set without visiting the failure set as training proceeds. Each run is an average over $1281$ episodes with initial states from the discretized state space. The statistics is computed from $50$ independent runs. The dashed red line specifies the success ratio of the chosen validation states based on the analytic reach-avoid set.}
    \label{fig:exp_zermelo_compare}
\end{figure}
Fig. \ref{fig:exp_zermelo_compare} shows the learned state value function as well as the training progress.
The state value function shows that RA conforms with the safety and liveness specification.
Thus, the state value is the same if, from that state, the control policy can navigate to the center of the target set.
In addition, the trajectory rollouts show that the policy keeps the distance from the failure set since that distance influences the outcome.
On the other hand, Sum can have a decent reach-avoid set only if the penalty is adequately specified ($\rho=1$).
Even if the state value is similar to RA's, the trajectory rollouts do often come close to the failure set, which is not desirable.
The training progress suggests that RA has a better (asymptotic) success ratio than Sum with an adequate penalty, and RA converges much faster than Sum with a low penalty ($\rho=0.1$).
RA provides an easy way to formulate the cost because it only requires $\targfunc$ and $\consfunc$ having the property in \eqref{eq:implicit_target} and \eqref{eq:implicit_failure}, while Sum requires careful tuning of the cost structure.


\begin{figure}[!t]
    \centering
    \includegraphics[width=0.9\columnwidth]{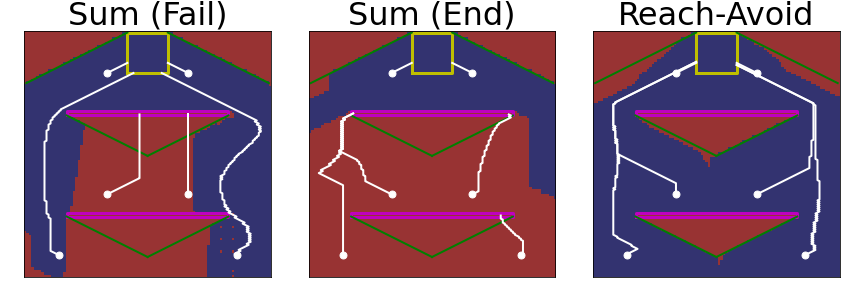}
    \caption{Binary reach-avoid outcome from executing the learned policies, for different objective formulations, after $400,000$ gradient updates. White lines show trajectory rollouts from six initial states. The red region indicates failed rollouts from the corresponding position states, while blue region indicates rollouts that reach the target. The first two figures correspond to reinforcement learning formulations with different rollout termination conditions. The last figure corresponds to our reach-avoid formulation.}
    \label{fig:exp_show}
\end{figure}
To show that an appropriate cost structure is sometimes challenging to find, we construct a new environment with two thin obstacles and one target while keeping the system dynamics the same.
Fig. \ref{fig:exp_show} shows the binary reach-avoid outcomes of trajectory rollouts.
The rightmost figure shows that RA has the correct rollout outcomes except for some regions around the top reach-avoid set boundary.
However, Sum with ``End'' termination only has the correct rollout outcomes for the region near the target.
In order to enforce safety to a larger degree in the Sum formulation, we terminate the training episode if the agent collides with the obstacles. We denote this new rollout termination condition as ``Fail'', shown in the leftmost figure.
We observe that Sum (Fail) has more correct outcomes than Sum (End), especially on the two sides of the environment.
However, in the central region, blocked by two thin obstacles, the learned policies fail to evade obstacles.
This result is possibly due to the abrupt value transition from the top boundary of the obstacles, so the negative cost is difficult to be propagated to the region under the obstacles.
The reason we observe these issues with the Sum formulation is that these objectives are only \emph{proxies} for the property we need our system to satisfy, whereas the reach-avoid objective is an exact encoding of it.
In other words, regardless of the environment configuration, the reach-avoid policy will be optimizing for ``the right thing''.
In this sense, the reach-avoid objective provides a more robust generalization across environments.

\subsection{3-D Dubins Car: Q-Learning as an Untrusted Oracle}
\begin{figure}[!t]
    \centering
    \includegraphics[width=0.8\columnwidth]{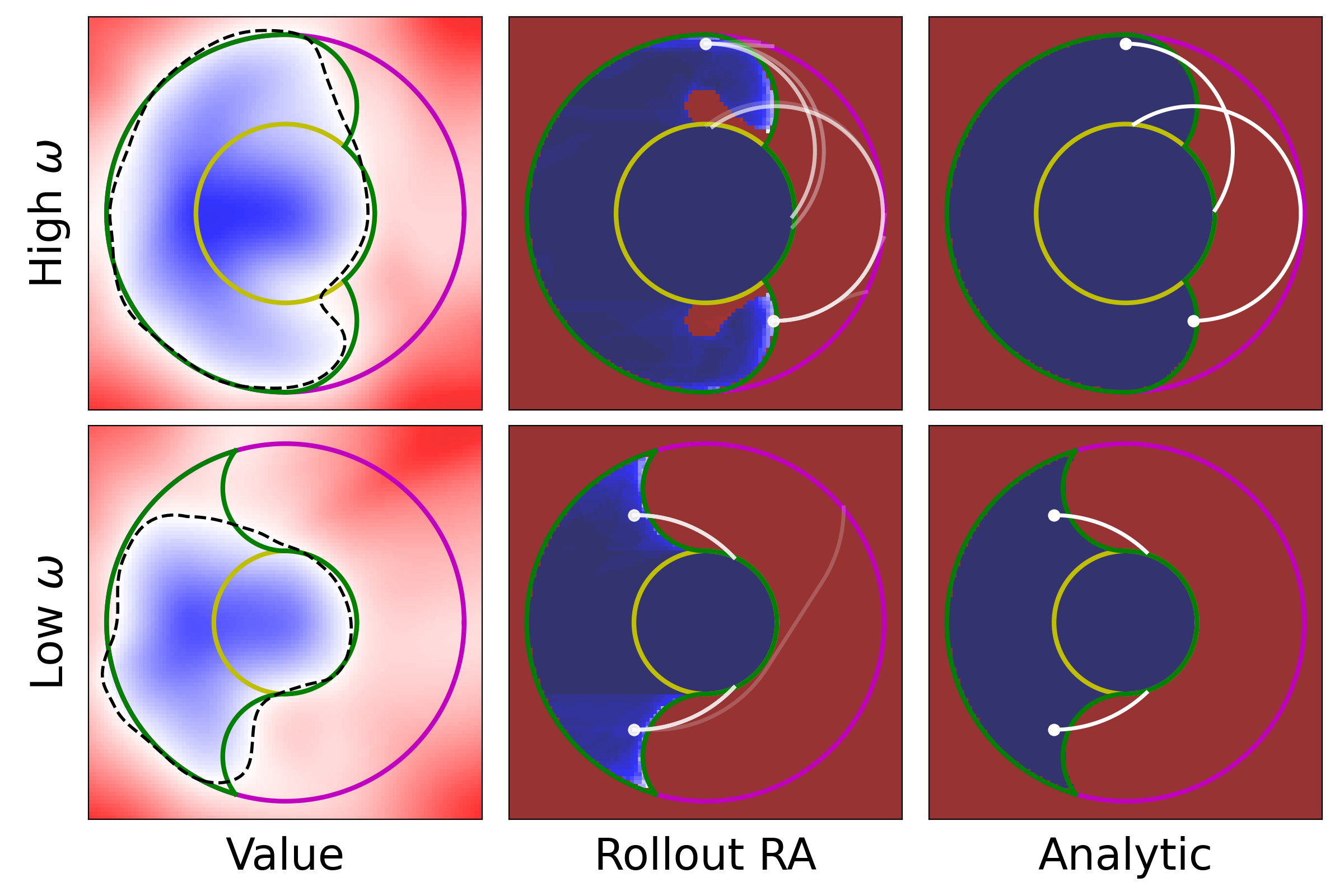}
    \caption{Value function learned after $400,000$ gradient updates. The reach-avoid set in the second column is the average success rate over $50$ models trained on different random seeds. White lines show trajectory rollouts of first $10$ models from two different initial states. The reach-avoid set in the third column is derived from geometric analysis and the trajectories are simply by constant angular velocity.}
    \label{fig:exp_car}
\end{figure}
The dynamics of the Dubins car are described by the following equation
\begin{IEEEeqnarray}{c}
\label{eq:dubinsCar_dynamics}
    \dot{x} = v \cos \theta,\
    \dot{y} = v \sin \theta,\ 
    \dot{\theta} = \ctrl,
\end{IEEEeqnarray}
where $\ctrl \in \{ \omega, 0, -\omega \}$.
In this environment, the state of the system is $\state = [x, y, \theta]$, with $x,y$ encoding position and $\theta$ encoding the heading.
In this environment we constrain the car to lie within the outer (magenta) circle and the goal is for the car to reach the inner (yellow) circle as shown in Fig.~\ref{fig:exp_car}.
The implicit surface functions here are defined as $\targfunc(\state) := \| \state \| - r$ and $\consfunc(\state) := \| \state \| - R$, where $r$ and $R$ are the inner and outer radii, respectively.
We construct two types of Dubins cars: (1) one with high turning rate ($r \geq 2v/\omega - R,\ v=0.5, \omega=0.833, R=1, r=0.5$) and (2) a second one with low turning rate ($r < 2v/\omega - R,\ v=0.5, \omega=0.667, R=1, r=0.4$). The results are shown in the first and the second row of Fig.~\ref{fig:exp_car}, respectively.

The left column of Fig.~\ref{fig:exp_car} shows a slice of the learned value function for a $0^\circ$ heading angle.
We first observe that the learned value function does not accurately approximate the analytic reach-avoid value function (rightmost column of Fig.~\ref{fig:exp_car}).
We take a closer look at the misclassified states by computing the confusion matrix between the predicted value function and the obtained outcome when rolling out the learned policy.
The condition is ``success" (``failure") when the learned policy rollout value is negative (positive) and the predicted condition is ``success" (``failure") when the DDQN value is negative (positive).
We obtain $6.6\%$ false-success rate (FSR)  
and $5.1\%$ false-failure rate (FFR)  
for the car with a high turning rate, while we get $7.9\%$ FSR and $2.5\%$ FFR for the car with a low turning rate.
Here we emphasize FSR because these are states from which we predict to reach the target set without entering the failure set, but the learned policy fails to deliver on this promise.

While the value function computed through deep Q-learning may not, by itself, constitute a reliable safety-liveness indicator, we can instead treat it as an \emph{untrusted oracle} encoding a best-effort reach-avoid policy.
Provided that we have the ability to simulate the dynamics of the system accurately (or with a nontrivial error bound), we can obtain an improved value prediction for any query state by rolling out the learned policy and observing its outcome.
In fact, this \emph{rollout value} has two practically useful properties: first, it is accurate (to the degree of fidelity allowed by the simulator); and second, it encodes the outcome achieved by the \emph{actually available} best-effort policy.
If this policy is substantially worse than the optimal one (due to poor learning, for example) the resulting rollout value will predict worse performance than could otherwise have been achieved, but it will nonetheless predict it \emph{accurately} for the learned policy. Even when small inaccuracies exist in the simulator, we can generally expect the FSR to be drastically lower than for the DDQN value prediction.

We minimize DDQN-predicted state-action values to obtain the DDQN ``untrusted oracle'' policy, and use this policy to obtain the rollout value, hence the approximate reach-avoid set.
The central column of Fig.~\ref{fig:exp_car} shows the average reach-avoid outcome across $50$ models trained on different random seeds if we use rollout value, while the rightmost column shows the ground truth analytic reach-avoid set obtained by geometric means.

This untrusted oracle approach enables us to translate our learned reach-avoid solution into a supervisory control law with a recursive feasibility guarantee, by employing a ``shielding'' scheme~\cite{bastani2019safe}.
Suppose that we start inside the rollout reach-avoid set (blue region in the central column of Fig.~\ref{fig:exp_car}).
Before we execute a candidate action from any policy (e.g. from a performance-driven reinforcement learning agent), we simulate a short trajectory forward to see if we would remain in the reach-avoid set after this action.
If the simulated rollout is successful, then we are free to execute the candidate action;
otherwise, we can use the reach-avoid policy to steer away from the reach-avoid set boundary, and thus maintain safety and liveness.
Using the untrusted oracle policy, we get a FFR of $4.8\%$ for high turn rate and $2.5\%$ for low turn rate, and a FSR of exactly $0$, by construction (thereby ensuring the success of shielding schemes provided we have an accurate simulator available at runtime).

\subsection{6-D Lunar Lander}

In this section we will showcase the algorithm on the lunar lander environment from OpenAI gym~\cite{brockman2016openai}. In this environment the state of the system is given by the six dimensional vector $\state := [x,y,\theta,\dot{x},\dot{y},\dot{\theta}]$, and the lander has four available actions at its disposal: to activate the left, right or main engine, or to apply no action.

Let us define the minimum signed $L_2$ distance between a point $[x,y]^T \in \mathbb{R}^2$ and a generic set $\mathcal{S}$ as
\begin{align}
    L_2(x,y;\mathcal{S}) =& ~ c \Big( \min_{[\hat{x},\hat{y}] \in \partial \mathcal{S}}     \sqrt{(x-\hat{x})^2 + (y-\hat{y})^2}\Big),
\end{align}
where $c = -1$ if $[x,y]^T \in \mathcal{S}$ and $c = +1$ otherwise, and $\partial \mathcal{S}$ is the boundary of the set.

The safety margin is defined by the minimum signed $L_2$ distance from the center of the lander to the moon surface, as well as the minimum distance to the left, right and top boundaries. Together, these four components define a polygon $\mathcal{P}_\mathcal{C} \subset \mathbb{R}^2$ which demarcates the constraint set. The \textit{safety margin} is thus defined as $g(s) = -L_2(x,y;\mathcal{P}_\mathcal{C})$.

The target set $\target$ is defined as a rectangular region within the constraint set (i.e. $\target \subset \mathcal{P}_\mathcal{C}$) that we want to reach without colliding. The \textit{target margin} is similarly defined as the minimum signed $L_2$ distance to $\target$, that is $l(s) = L_2(x,y;\target)$.


In our experiments we run the environment for 5 million gradient updates. Fig.~\ref{fig:slices_values_function} shows slices of the value function for different values of $v_x$ and $v_y$, while keeping $\theta$ and $\dot\theta$ constant. In this case the system is too high-dimensional to be able to compare the value function to a ground-truth. However, Fig. \ref{fig:slices_values_function} shows the correct qualitative behavior of our solution, with leftward and rightward initial speeds causing the value function to extend horizontally from the boundary and the obstacle. Similarly, initial vertical speeds cause the value function to be positive near the bottom and top of the environment. For these experiments we find the FSR and FFR to be $20 \%$ and $4\%$ respectively.

\begin{figure}[!t]
    \centering
    \includegraphics[width=0.85\columnwidth]{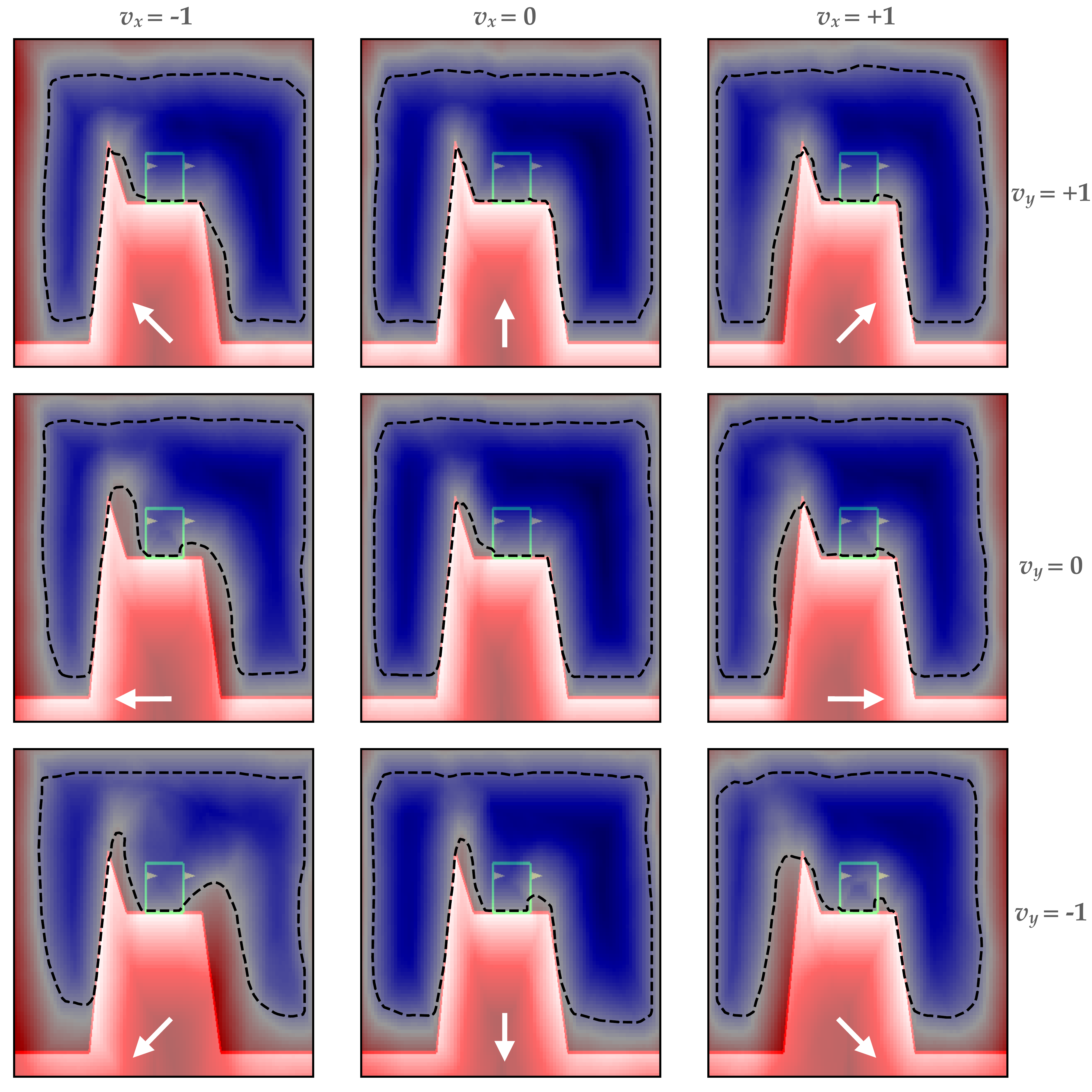}
    \caption{Slices of the value function for different values of $v_x$ and $v_y$, with $\theta = 0$ and $\dot{\theta} = 0$. Top, middle and bottom rows correspond to $v_y=1,0,-1$ respectively. First, second and last column represent $v_x=-1,0,1$ respectively. The dashed line denotes the zero level set, and the arrows denote the direction of the initial speed.}
    \label{fig:slices_values_function}
\end{figure}

\subsection{6-D Attack-Defense Game with Two Dubins Cars}

\begin{figure}[!t]
    \centering
    \begin{minipage}[c]{0.3\columnwidth}
        \centering
        \includegraphics[width=\columnwidth]{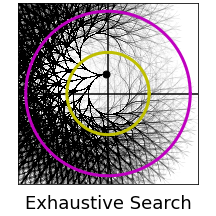}
        \subcaption{}
    \end{minipage}
    \begin{minipage}[c]{0.6\columnwidth}
        \centering
        \includegraphics[width=\columnwidth]{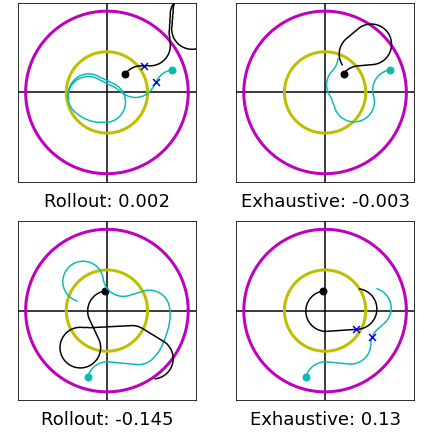}
        \subcaption{}
    \end{minipage}
    \caption{Validation of the \textit{defender's} policy. (a): exhaustive trajectories. (b): DDQN rollout and the worst case in exhaustive trajectories. First row: false failure; Second row: false success. Black lines show the defender's trajectories, while green lines show the attacker's trajectories. Blue x's specify the instant when the attacker is captured by the defender.}
    \label{fig:exp_pe}
\end{figure}
In this experiment, we consider attack-defense games between two identical Dubins cars.
One car, \textit{the attacker}, wants to reach the yellow target without hitting the magenta failure set boundary.
On the other hand, the other car, \textit{the defender}, attempts to catch the attacker.
We consider the game is won if the attacker manages to reach their goal before they are caught by the defender. An example is shown in Fig. \ref{fig:exp_pe}.
In the following paragraphs, we will use subscript $A$ to denote variables related to the attacker and $D$ to the defender.
The state of the system in this experiment is $\state = (\state_A, \state_D)$ and the dynamics of the Dubins car follow \eqref{eq:dubinsCar_dynamics} with $v=0.75, \omega=3, R=1, r=0.5$.
The target set and failure set for this experiment are defined below.
\begin{IEEEeqnarray}{c}
\label{eq:car_pe_target}
   \target := \{ (\state_A, \state_D) \ | \ \state_A \in \target_A \}, \\
   \failure := \{ (\state_A, \state_D) \ | \ \state_A \in \failure_A \vee (\state_A, \state_D) \in \failure_J \},
\end{IEEEeqnarray}
where $\failure_J := \{ (\state_A, \state_D) \ | \ \| \state_A - \state_D \|_2 \leq \beta \}$ and $\beta$ is the capture range.
We use $\beta = 0.25$ in this experiment.
Their implicit surface functions are as follows.
\begin{IEEEeqnarray}{c}
    \targfunc(\state) = \| \state_A \|_2 - r 
    \nonumber \\
    \consfunc(\state) := \max \Big\{ \| \state_A \|_2 - R, \ \beta - \| \state_A - \state_D \|_2 \Big\} \nonumber.
\end{IEEEeqnarray}
We learn the minimax player policies by replacing $\min_{\ctrl \in \cset} \valfunc( \state_+^{\ctrl})$ in \eqref{eq:DRABE} with $\min_{\ctrl_A \in \cset} \max_{\ctrl_D \in \cset} \valfunc( \state_+^{\ctrl_A, \ctrl_D})$.
With this modification, the output of the DDQN becomes a matrix indexed by controls from the attacker and defender. 
We evaluate the minimax player policies by DDQN values in two ways: (1) \textit{estimation error}, how well we learn from the data, and (2) \textit{approximation error}, how closed the minimax policy is to the (conservatively) optimal policy.

For estimation error, we sample $225$ attacker positions uniformly in the ring and $225$ defender positions uniformly inside the constraint set.
In addition, we uniformly sample $15$ heading angles in $[0,\ 2\pi]$ for both attacker and defender.
For these $15^6$ state samples, we execute the policy by DDQN values for $100$ time steps. 
We compare this rollout value with the DDQN predicted minimax value and got $6.7\%$ FFR and $23.3\%$ FSR.

For approximation error, we want to check if the defender's policy is optimal by two-round validation.
In each round, we divide $50$ time steps into ten intervals, and within each interval, the defender adopts the same action.
Therefore, we have $3^{10}$ different action sequences, as shown in Fig. \ref{fig:exp_pe} (a).
We obtain the rollout by letting the attacker follow the DDQN policy, but the defender follows the artificial action sequences.
We rank exhaustive trajectories by (1) crossing the constraint or captured, (2) unfinished, and (3) succeeding.
We pick the worst trajectory in the attacker's viewpoint and record the rollout value as the exhaustive value.
If the worst trajectory from the first-round is unfinished, we conduct the second-round validation from the end-point of that trajectory.

We sample $500$ initial states from the states with negative and positive DDQN rollout value, and we obtain $33\%$ and $1.2\%$ error rate, respectively.
Fig. \ref{fig:exp_pe} (b) shows trajectories by DDQN rollout and exhaustive search with their payoff, which is defined in \eqref{eq:cost_functional_reach_avoid}.
The first row of Fig. \ref{fig:exp_pe} (b) shows a rare example in which the attacker succeeds in validation but fails in the DDQN rollout.
This discrepancy is due to restricting the defender to maintaining the same action within each short time interval.
However, we note that we have very low error rate for states of positive DDQN rollout value.
In this case, even though the defender fails in all exhaustive trajectories, the attacker barely escapes with $-0.02$ average exhaustive rollout value.
The second row of Fig. \ref{fig:exp_pe} (b) shows that the attacker fails in validation but succeeds in the DDQN rollout.
We make the following two observations:
(1) the defender in the DDQN rollout chooses to chase the attacker right away, while in validation, the defender chooses to block the path for the attacker to enter into the target set.
This result suggests that the learned policy may not be optimal.
(2) it seems possible that the attacker can reach the target set before the defender can get in the way.
However, the defender adopts some unintuitive action sequence, e.g., a clockwise turn in the bottom left.
This action results in a infrequently visited joint state, at which the neural network does not produce an accurate approximate value.
This indicates that our ``exhaustive'' validation is able to expose vulnerabilities in the learned attacker policy.
\section{Conclusion} \label{sec:conclusion}
In this paper we have presented a learning approach for computing solutions to reach-avoid optimal control problems. We first presented the optimality condition for the value function of reach-avoid problems and how it does not naturally induce a contraction mapping in the space of value functions. Upon introducing a discount term, we then constructed a discounted version of the optimality condition, or DRABE, which does induce a contraction. Notably, we show that using the DRABE operator we are able to recover conservative approximations of the exact reach-avoid set for any $\gamma \leq 1$. We then proceeded to use this discounted update in tabular and deep Q-learning in order to compute approximately optimal reachable sets and controllers for a variety of non-linear systems. Our results are a first step in showing how Hamilton-Jacobi reachability can be coupled with reinforcement learning methodologies to yield solutions for problems where one wishes to optimize for both safety and liveness.

\appendices
\renewcommand{\arraystretch}{1.15}
\begin{table*}[!t]
\centering
\caption{The hyper-parameters used in different environments. PP: point particle, DC: Dubins car, LL: Lunar Lander, AD: attack-defense.}
\begin{threeparttable}
    \begin{tabular}{ |l||c|c|c|c|c| } 
        \hline
        \multirow{2}{*}{Environment} & \multicolumn{2}{c|}{PP} & \multirow{2}{*}{DC} & \multirow{2}{*}{LL} & \multirow{2}{*}{AD}
        \\ \cline{2-3}
        & Fig. \ref{fig:exp_convergentFamily} & Fig. \ref{fig:exp_zermelo_compare}, \ref{fig:exp_show} & & &
        \\ \hline
        state dimension & \multicolumn{2}{c|}{$2$} & $3$ & $6$ & $6$
        \\ \hline
        action set dimension & \multicolumn{2}{c|}{$3$} & $3$ & $4$ & $9$
        \\ \hline
        $\#$ gradient updates ($T$) & $12$M & \multicolumn{2}{c|}{$400, 000$} & $5$M & $4$M
        \\ \hline
        NN architecture & \multicolumn{3}{c|}{($100$, $20$)} & \multicolumn{2}{c|}{($512$, $512$, $512$)}
        \\ \hline
        NN activation & \multicolumn{5}{c|}{$\mathrm{Tanh}$}
        \\ \hline
        learning rate & \multicolumn{5}{c|}{$\max \{ 0.001 \times 0.8^{\lfloor 20x/T \rfloor},\ 0.0001 \}$\tnote{*}}
        \\ \hline
        optimizer & \multicolumn{2}{c|}{Adam} & \multicolumn{3}{c|}{AdamW}
        \\ \hline
        discount rate ($\gamma$) & \tnote{\S} & \multicolumn{2}{c|}{$0.9999$} & \multicolumn{2}{c|}{\tnote{\S}}
        \\ \hline
        exploration-exploitation & \multicolumn{5}{c|}{$\max \{ 0.95 \times 0.6^{\lfloor 20x/T \rfloor},\ 0.05 \}$\tnote{*}}
        \\ \hline
        DDQN soft-update & \multicolumn{5}{c|}{$0.01$}
        \\ \hline
        replay buffer size & \multicolumn{3}{c|}{$10000$} & \multicolumn{2}{c|}{$50000$}
        \\ \hline
        initialization\tnote{\dag} & \multicolumn{2}{c|}{without} & $\max\{ \targfunc, \consfunc \}$ & $\consfunc$ & $\max\{ \targfunc, \consfunc \}$
        \\ \hline
    \end{tabular}
    \begin{tablenotes}
        \footnotesize
        \item[*] $x$: the number of updates so far, $T$: maximal number of updates.
        \item[\dag] with uniform samples in the $x-y$ plane
        \item[\S] $\min \{ 1 - 0.2 \times 0.5^{\lfloor 20x/T \rfloor},\ 0.999999 \}$
    \end{tablenotes}
\end{threeparttable}
\label{tab:hyper-parameter}
\end{table*}
\section{Fixed-Point RABE Derivation and Other Proofs}
\label{appendix_fixed_point}

In what follows we show how to use Bellman's principle of optimality to obtain the optimality condition for the value function:
\begin{align}
\label{eq:variational_one_}
    \valfunc(&\state) = \inf_{\csig}  \min_{\tau \in \{0,1,...\}}  \max \big\{ l(\traj_{\state}^{\csig}(\tau)) ,  \max_{\kappa \in \{0,\tau\}} g(\traj_{\state}^{\csig}(\kappa)) \big\} \nonumber\\ 
     &= \min \Big\{ \max \big\{ l(\state) ,  g(\state) \big\}  , \inf_{\csig} \mathbf{\min_{\tau \in \{1,...\}}} \\
     & \qquad \qquad \max \big\{ l\big(\traj_{\state_1}^{\mathbf{\csig}}(\tau)\big) ,  \max_{\kappa \in \{0,\tau\}} g\big(\traj_{\state}^{\csig}(\kappa)\big) \big\} \Big\} \nonumber\\ 
     \begin{split}
     &= \min \Big\{ \max \big\{ l(\state) ,  g(\state) \big\} , \inf_{\csig} \min_{\tau \in \{1,...\}} \\ & \qquad \qquad \max \big\{ g(\state), \max\{l(\traj_{\state_1}^{\csig}(\tau)), \max_{\kappa \in \{1,\tau\}} g(\traj_{\state_1}^{\mathbf{\csig}}(\kappa)) \}\big\}\Big\} 
     \end{split} \nonumber\\
     \begin{split}
     &= \min \Big\{ \max \big\{ l(\state) ,  g(\state) \big\} , \inf_{u}  \max \big\{g(\state),\\ & \qquad \qquad \inf_{\csig_{1}} \min_{\tau \in \{1,...\}}  \max\{ l(\traj_{\state_1}^{\csig_1}(\tau)), \max_{\kappa \in \{1,\tau\}} g(\traj_{\state_1}^{\csig_1}(\kappa)) \}\big\}\Big\} 
     \end{split} \nonumber\\
     &= \min \{ \max \{ l(\state) ,  g(\state) \} , \inf_{u}  \max \{g(\state), \valfunc(\state_+^\ctrl)\}\} \nonumber\\ 
     &= \min \{ \max \{ l(\state) ,  g(\state) \} , \max \{g(\state), \inf_{u} \valfunc(\state_+^\ctrl)\}\}, \nonumber
\end{align}
where $\state_+^\ctrl = \state + f(\state, \ctrl) \Delta t$, and $\csig_1$ refers to the control signal without the first control input. Noting that ${\min\{\max\{a,b\},\max\{b,c\}\} = \max\{b,\min\{a,c\}\}}$ we obtain the fixed-point RABE:

\begin{IEEEeqnarray}{rl}
    \valfunc(\state) = \max\Big\{ \consfunc(\state), \ & \min\big\{\targfunc(\state),
    \inf_{\ctrl \in \cset} \valfunc(\state_+^\ctrl )~\big\}~\Big\}.
\end{IEEEeqnarray}

Now we prove the proposed discounted (D)RABE induces a contraction mapping under the supremum norm, as described in Proposition~\ref{prop:contraction}.
\begin{proof}[Proof of Proposition~\ref{prop:contraction}]
Observing that $|\max \{a, b\} - \max \{a, c\}| \leq |b-c|$, $\forall a,b,c\in\mathbb{R}$, we have
\begin{IEEEeqnarray}{rl}
\Big| & \max \big\{
    \min \{ \
            \min_{\ctrl \in \cset} \valfunc_\gamma(\state_+^\ctrl),\ 
            \targfunc(\state)
        \},\
        \consfunc(\state) 
    \big\} \ -
\nonumber \\ & 
    \max \big\{
        \min \{ \
            \min_{\ctrl \in \cset} \tilde\valfunc_\gamma(\state_+^\ctrl),\ 
            \targfunc(\state) 
        \},\
        \consfunc(\state)
    \big\}
\Big| 
\nonumber \\
\leq \ & \left| 
    \min \{ \
        \min_{\ctrl \in \cset} \valfunc_\gamma( \state_+^\ctrl),\ 
        \targfunc(\state) \} -
    \min \{ \
        \min_{\ctrl \in \cset} \tilde\valfunc_\gamma( \state_+^\ctrl),\ 
        \targfunc(\state) \}
\right|
\nonumber \\
\leq \  & \left| 
\min_{\ctrl \in \cset} \valfunc_\gamma( \state_+^\ctrl) -
\min_{\ctrl \in \cset} \tilde\valfunc_\gamma( \state_+^\ctrl)
\right|.
\nonumber
\label{eq:tmp1}
\end{IEEEeqnarray}
Without loss of generality, suppose $\min_{\ctrl \in \cset} \valfunc_\gamma( \state_+^\ctrl) >
\min_{\ctrl \in \cset} \tilde\valfunc_\gamma( \state_+^\ctrl)$ and let $\ctrl^* := \arg \min_{\ctrl \in \cset} \tilde\valfunc_\gamma( \state_+^\ctrl)$.
Thus, $\valfunc_\gamma( \state_+^{\ctrl^*}) \geq \min_{\ctrl \in \cset} \valfunc_\gamma( \state_+^\ctrl) \geq \tilde\valfunc_\gamma( \state_+^{\ctrl^*})$.
Then, we have
\begin{IEEEeqnarray}{ll}
& \Big| B_\gamma [\valfunc_\gamma](\state) - B_\gamma [\tilde\valfunc_\gamma](\state) \Big|
\nonumber \\
\leq & \gamma\ \Big| 
    \min_{\ctrl \in \cset} \valfunc_\gamma( \state_+^\ctrl) -
    \min_{\ctrl \in \cset} \tilde\valfunc_\gamma( \state_+^\ctrl)
\Big| 
\leq \gamma\ \Big| 
    \valfunc_\gamma( \state_+^{\ctrl^\star}) -
    \tilde\valfunc_\gamma ( \state_+^{\ctrl^\star})
\Big| 
\nonumber \\
\leq & \gamma\ \max_{\ctrl \in \cset} \Big| 
    \valfunc_\gamma( \state_+^\ctrl) - \tilde\valfunc_\gamma( \state_+^\ctrl)
\Big|
\leq \gamma\ \max_{\state \in \xset} \Big| 
    \valfunc_\gamma(\state) - \tilde\valfunc_\gamma(\state)
\Big|.
\nonumber
\qedhere
\end{IEEEeqnarray}
\end{proof}

Further, we prove the convergence result for the Reach-Avoid Q-Learning Scheme, as presented in Proposition~\ref{prop:ra_q_learning}.
\begin{proof}[Proof of Proposition~\ref{prop:ra_q_learning}]
    Our proof follows from the general proof of Q-learning convergence for finite-state, finite-action Markov decision processes presented in \cite{Tsitsiklis1994}.
    Our transition dynamics $\dyn$, initialization and policy randomization, and learning rate $\alpha_k$ satisfy Assumptions 1, 2, and 3 in \cite{Tsitsiklis1994} in the standard way.
    The only critical difference in the proof is the contraction mapping, which we obtain under the supremum norm by Proposition \ref{prop:contraction}:
    with this, Assumption 5 in \cite{Tsitsiklis1994} is met, granting convergence of Q-learning by Theorem~3~in~\cite{Tsitsiklis1994}.
\end{proof}

In addition, we prove that the fixed-point solution of DRABE is the solution of RABE when the discount rate approaches $1$, as stated in Proposition~\ref{prop:convergent_value}.
\begin{proof}[Proof of Proposition~\ref{prop:convergent_value}]
Let $\targfunc_t$ and $\consfunc_t$ stand for the target margin and safety margin at $t$-th time step of a discrete-time trajectory.
The discounted value of this trajectory is
\begin{IEEEeqnarray}{c}
\outcome_\gamma^{\csig}(\state) \ = \
(1-\gamma) \max \{ \targfunc_0, \consfunc_0 \} + \gamma \max \Big\{
    \consfunc_0,
    \min \big\{
        \targfunc_0, 
\nonumber \\
\ \ (1-\gamma) \max \{ \targfunc_1, \consfunc_1 \} + 
        \gamma \max\{ \consfunc_1, \min \{ \targfunc_1, \cdots \}
    \big\}
\Big\},
\end{IEEEeqnarray}
which is the explicit form of the objective minimized in \eqref{eq:DRABE}.
By taking $\gamma \rightarrow 1$, we recover
\begin{IEEEeqnarray}{rcl}
    \lim_{\gamma \rightarrow 1} & \outcome_\gamma^{\csig}(\state)
    = & \max \Big\{
        \consfunc_0, \min \big\{ \targfunc_0, \max \big\{ \consfunc_1, \min \{ \targfunc_1, \cdots \} \big\} \big\} \Big\}
    \nonumber \\
    = & \min & \big\{
        \max \{ \consfunc_0, \targfunc_0 \},
        \max \big\{
            \consfunc_0, 
            \consfunc_1,
            \min \{ \targfunc_1, \cdots \} 
        \big\}
    \big\}
    \nonumber \\
    = & \min & \big\{
        \max \{ \consfunc_0, \targfunc_0 \}, \
        \max \{ \consfunc_0, \consfunc_1, \targfunc_1 \},
    \nonumber \\
    & & \ \ \max \{
        \consfunc_0, \consfunc_1, \consfunc_2,
        \min \{ \targfunc_2, \cdots \} \}
    \big\}
    \nonumber \\
    = & \min_{\tau \in \{0, 1, \cdots\}} & \max \Big\{ \targfunc_\tau, \ \max_{\kappa \in \{0, 1, \cdots, \tau \}} \consfunc_\kappa \Big\}
    = \outcome^{\csig}(\state). \nonumber
\end{IEEEeqnarray}
Thus,
\begin{IEEEeqnarray}{c}
    \lim_{\gamma \rightarrow 1} \valfunc_\gamma(\state) = \lim_{\gamma \rightarrow 1} \min_{\csig \in \csigset} \outcome_\gamma^{\csig}(\state) = \min_{\csig \in \csigset} \outcome^{\csig}(\state) = \valfunc(\state). \nonumber \qedhere
\end{IEEEeqnarray}
\end{proof}

Finally, we prove that any choice of $\gamma$ induces a conservative approximation to the exact reach-avoid set, as stated in Theorem~\ref{thm:underapproximation_convergent_family}.
\begin{proof}[Proof of Theorem~\ref{thm:underapproximation_convergent_family}]
For brevity, let $\ctrl^* := \arg \min_\ctrl \valfunc( \state_+^\ctrl)$.
From \eqref{eq:cost_functional_reach_avoid} and \eqref{eq:value_function_reach_avoid}, we know $\max\{ \targfunc(\state), \consfunc(\state) \}$ is always bigger or equal than $\valfunc(\state)$, so we have $\valfunc_{\gamma_1} (\state) \geq \valfunc_{\gamma_2} (\state)$ for any $\gamma_1 \leq \gamma_2$.
Here we enumerate all cases as follows
\begin{enumerate}
    \item $\valfunc( \state_+^{\ctrl^*}) \geq \targfunc(\state)$:
        \begin{IEEEeqnarray}{l}
           \max\{ \targfunc(\state), \consfunc(\state) \} - \max \big\{ 
            \min \big\{
                \valfunc( \state_+^{\ctrl^*}),\ 
                \targfunc(\state) \big\},\
            \consfunc(\state) \big\} 
            \nonumber \\
            = \ \max\{ \targfunc(\state), \consfunc(\state) \} -\max\{ \targfunc(\state), \consfunc(\state) \} = 0.
            \nonumber
        \end{IEEEeqnarray}
    \item $\consfunc(\state) \geq \valfunc( \state_+^{\ctrl^*})$ and $\consfunc(\state) \geq \targfunc(\state)$:
        \begin{IEEEeqnarray}{l}
           \max\{ \targfunc(\state), \consfunc(\state) \} - \max \big\{ 
            \min \big\{
                \valfunc( \state_+^{\ctrl^*}),\ 
                \targfunc(\state) \big\},\
            \consfunc(\state) \big\}
            \nonumber \\
            = \ \consfunc(\state) - \consfunc(\state) = 0.
            \nonumber
        \end{IEEEeqnarray}
    \item $\targfunc(\state) > \valfunc( \state_+^{\ctrl^*}) \geq \consfunc(\state)$:
        \begin{IEEEeqnarray}{l}
           \max\{ \targfunc(\state), \consfunc(\state) \} - \max \big\{ 
            \min \big\{
                \valfunc( \state_+^{\ctrl^*}),\ 
                \targfunc(\state) \big\},\
            \consfunc(\state) \big\}
            \nonumber \\
            = \ \targfunc(\state) - \valfunc( \state_+^{\ctrl^*}) > 0.
            \nonumber
        \end{IEEEeqnarray}
    \item $\targfunc(\state) > \consfunc(\state) > \valfunc( \state_+^{\ctrl^*})$:
        \begin{IEEEeqnarray}{l}
           \max\{ \targfunc(\state), \consfunc(\state) \} - \max \big\{ 
            \min \big\{
                \valfunc( \state_+^{\ctrl^*}),\ 
                \targfunc(\state) \big\},\
            \consfunc(\state) \big\}
            \nonumber \\
            = \ \targfunc(\state) - \consfunc(\state) > 0.
            \nonumber
        \end{IEEEeqnarray}
\end{enumerate}
Since $\forall \gamma_1 \leq \gamma_2,\ \forall \state : \valfunc_{\gamma_1} (\state) \leq 0 \rightarrow \valfunc_{\gamma_2} (\state) \leq 0$, we have $\reachavoid_{\gamma_1} \subseteq \reachavoid_{\gamma_2}$.
Also, taking the limit of \eqref{eq:DRABE} as $\gamma \rightarrow 1$ we recover $\lim_{\gamma \rightarrow 1} \valfunc_\gamma (\state) = \valfunc (\state) $.

Finally, we consider (Hausdorff) set convergence:
we seek to show $\forall\delta>0, \exists\gamma\in[0,1) \mid d_H(\reachavoid_\gamma,\reachavoid)<\delta$.
Since $\reachavoid_\gamma\subseteq\reachavoid$, we have: ${d_H(\reachavoid_\gamma,\reachavoid) = \sup_{x\in\reachavoid} \inf_{y\in\reachavoid_\gamma}\|x - y\|}$ (for an arbitrary norm $\|\cdot\|$ on $\xset$).
Provided that $\valfunc$ is nowhere locally constant on its zero level set and the interior of $\reachavoid$ is nonempty, we have that for any $\delta>0$, there exists $\epsilon>0$ such that near any $x$ with $\valfunc(x)\le 0$ there exists a point $y$, $\|x-y\|\le\delta$ for which $\valfunc(y) < -\epsilon$.
From the limit of~\eqref{eq:DRABE}, we also have that for any such $\epsilon$, there exists a sufficiently large $\gamma\in[0,1)$ for which $|\valfunc_{\gamma}(y)-\valfunc(y)|<\epsilon$, and hence $\valfunc_{\gamma}(y) < 0$.
Therefore, for any $\delta>0$ we can find $\gamma \in [0,1)$ such that any point $x\in\reachavoid$ has a corresponding point $y\in\reachavoid_\gamma$ within distance $\delta$.
This concludes the proof.
\end{proof}
\section{Practical Training Considerations}
\label{appendix_traing}

Table \ref{tab:hyper-parameter} lists the state dimension, action set dimension, DDQN hyper-parameters and NN training method in each environment.
Among these hyper-parameters, we find that the discount rate, memory buffer size and NN architecture are the most crucial ones.
If we start with very high $\gamma$, say 0.99, the convergence becomes extremely slow and DDQN cannot learn the value function well.
For more complicated environment, we need deeper NN and larger memory buffer, so we have stronger representability and more diverse state-action pairs.




\balance


\printbibliography

\end{document}

